\documentclass[11pt]{article}

\usepackage[a4paper,margin=1in]{geometry}
\usepackage[T1]{fontenc}
\usepackage[utf8]{inputenc}
\usepackage{lmodern}
\usepackage{microtype}

\usepackage{amsmath,amssymb,amsthm,mathtools}
\usepackage{bm,bbm}

\usepackage{graphicx}
\DeclareGraphicsExtensions{.pdf,.png,.jpg}
\graphicspath{{figs/}}
\usepackage{subcaption}
\usepackage{booktabs}
\usepackage{float}

\usepackage{algorithm}
\usepackage{algpseudocode}

\usepackage{xcolor}
\usepackage[hidelinks]{hyperref}
\hypersetup{
  colorlinks=true,
  linkcolor=blue,
  citecolor=blue,
  urlcolor=blue
}
\usepackage[nameinlink,capitalize,noabbrev]{cleveref}
\usepackage[numbers,sort&compress]{natbib}

\usepackage{tikz}
\usetikzlibrary{positioning}

\numberwithin{equation}{section}
\newtheorem{theorem}{Theorem}[section]
\newtheorem{lemma}[theorem]{Lemma}

\theoremstyle{definition}

\theoremstyle{remark}

\newcommand{\gap}{\operatorname{gap}}
\newcommand{\op}{\mathrm{op}}

\title{Spectral Identifiability for Interpretable Probe Geometry}

\author{William Hao-Cheng Huang\\
Taiwan Semiconductor Manufacturing Company (TSMC), Hsinchu, Taiwan\\
\texttt{williamhuang0709@gmail.com}
}
\date{}

\newcommand{\E}{\mathbb{E}}
\newcommand{\R}{\mathbb{R}}
\DeclareMathOperator{\sign}{sign}

\begin{document}

\maketitle

\begin{abstract}
Linear probes are widely used to interpret and evaluate neural representations, yet their reliability remains unclear, as probes may appear accurate in some regimes but collapse unpredictably in others. 
We uncover a spectral mechanism behind this phenomenon and formalize it as the \textit{Spectral Identifiability Principle (SIP)}—a verifiable Fisher-inspired condition for probe stability. 
When the eigengap separating task-relevant directions is larger than the Fisher estimation error, the estimated subspace concentrates and accuracy remains consistent, whereas closing this gap induces instability in a phase-transition manner. 
Our analysis connects eigengap geometry, sample size, and misclassification risk through finite-sample reasoning, providing an interpretable diagnostic rather than a loose generalization bound. 
Controlled synthetic studies, where Fisher quantities are computed exactly, confirm these predictions and show how spectral inspection can anticipate unreliable probes before they distort downstream evaluation.
\end{abstract}


\section{Introduction}

Understanding what neural networks learn requires reliable probes of internal representations.  
Probing methods—typically lightweight classifiers trained on frozen features—are widely used but often unstable: their reported accuracy can fluctuate sharply across runs, datasets, or random seeds \citep{alain2017understanding,hewitt2019designing}.  
This variability raises a central question: \emph{when can probe results be trusted?}

\paragraph{From heuristics to verifiable analysis.}  
Most existing probing techniques offer useful but post-hoc evaluations that measure interpretability only after training.  
We instead seek a verifiable, pre-deployment condition that anticipates probe reliability before instability arises.  
Such a criterion enables a principled approach to analyzing representational stability rather than relying solely on empirical variance or heuristics.

\paragraph{Spectral view of probe reliability.}  
Our approach is grounded in spectral perturbation theory.  
By examining how the empirical Fisher operator deviates from its population counterpart, we derive a finite-sample condition that predicts when the discriminative subspace remains identifiable.  
This spectral view reframes interpretability evaluation as a falsifiable problem—one that can be analyzed and tested rather than inferred from post-hoc accuracy trends.

\paragraph{Contributions.}  
Our work makes three main contributions:
\begin{itemize}
    \item[(i)] \textbf{Methodology.} We propose the \textit{Spectral Identifiability Principle (SIP)}, a verifiable criterion for assessing probe reliability through eigengap geometry.
    \item[(ii)] \textbf{Theory.} We establish a finite-sample and spectral connection between Fisher operator estimation, subspace deviation, and misclassification risk, yielding an interpretable stability condition.
    \item[(iii)] \textbf{Empirics.} We validate the theory through controlled synthetic studies with tunable eigengaps and heavy tails, confirming the predicted phase transitions.
\end{itemize}
Together, these results bridge spectral theory and practical interpretability, providing a compact framework for reliable probe evaluation.


\section{Related Work}
\label{sec:related}

\paragraph{Background.}
Probing methods are a central tool for analyzing neural representations \citep{alain2017understanding,hewitt2019designing,belinkov2022probing}.  
They train lightweight classifiers on frozen features to test what information a model encodes, but their reliability remains debated—probe accuracies often vary sharply across seeds or datasets.  
This raises a central question: \emph{under what conditions can probe evaluations be trusted?}

\paragraph{Information-theoretic and causal probes.}
Recent work has sought to formalize probe evaluation through information-theoretic or causal diagnostics, such as minimum description length (MDL) \citep{voita2020information,pimentel2020itp,pimentel2021bayesian} and amnesic or counterfactual interventions \citep{elazar2021amnesic,ravfogel2020inlp}.  
These frameworks measure sufficiency or causal relevance but remain post-hoc and heuristic, lacking falsifiable criteria for stability.  
Surveys highlight this limitation and call for pre-deployment reliability diagnostics \citep{belinkov2022probing}.

\paragraph{Spectral and Fisher analyses.}
Spectral analysis connects stability to eigengaps and perturbation bounds \citep{yu2015use,tropp2012user}.  
In machine learning, Fisher and Hessian spectra have been used to study curvature, conditioning, and optimization stability \citep{martens2020new,sagun2018empirical,kunstner2019limitations,wu2024improved,deb2025fishersft}.  
However, these studies focus on training dynamics rather than how estimation errors translate into \emph{classification risk}—the key determinant of probe reliability.

\paragraph{Position of this work.}
Our framework bridges these perspectives through a finite-sample, verifiable spectral law—the \emph{Spectral Identifiability Principle (SIP)}—which links Fisher estimation error to subspace stability and misclassification risk.  
SIP complements information-theoretic and causal approaches by offering a simple, testable diagnostic grounded in classical perturbation theory, providing a unified and verifiable condition for trustworthy probing.


\section{Preliminaries}
\label{sec:prelim}

In this section, we formalize the key components needed for understanding probe stability in the context of representation learning. We analyze how spectral geometry and estimation error jointly determine probe generalization, laying the foundation for the verifiable condition in later sections. These preliminaries will set up the precise link between spectral estimation and classification performance, which is the foundation for proving the verifiable condition for probe reliability in later sections.

\paragraph{Data and representation.}
We observe samples $(X,Y)$ with $h(X)\in\R^d$, where $Y$ is a binary label.  
A linear probe takes the form
\[
    f(x) = \sign(w^\top h(x)), \quad w\in\R^d .
\]
Throughout, we fix $h(\cdot)$ and study probe behavior conditioned on this representation.  
Understanding how the features in $h(X)$ relate to the label is central to explaining probe performance and stability.

\paragraph{Fisher operator and subspace.}
Following recent representation-learning literature, we define
\[
    \Gamma = \E[h(X)h(X)^\top]
\]
and refer to it as the \emph{Fisher operator}.  
Although not the classical Fisher information, this uncentered second moment captures the discriminative geometry exploited by probes, and we therefore refer to it as the \emph{Fisher operator}. In representation learning, this operator directly encodes the alignment between $h(X)$ and the label.\\  
Using covariance instead would remove label-related structure, whereas the uncentered moment preserves task alignment.
Intuitively, the top-$k$ eigenspace of $\Gamma$ captures the directions where the representation varies most strongly with respect to the task, serving as the candidate discriminative subspace.  

Let $U\in\R^{d\times k}$ denote the top-$k$ eigenspace of $\Gamma$, separated by the eigengap
\[
    \gap(\Gamma) = \lambda_k(\Gamma) - \lambda_{k+1}(\Gamma) > 0 .
\]
The existence of this gap is essential for probe stability: it ensures that the task-relevant subspace is well-defined and recoverable from the data.

\paragraph{Empirical estimate.}
In practice, we do not have access to the true Fisher operator $\Gamma$.  
Instead, we estimate it from finite samples.  
The \emph{empirical Fisher operator} is
\[
    \widehat\Gamma = \tfrac1n \sum_{i=1}^n h(X_i)h(X_i)^\top ,
\]
with estimation error
\[
    \Delta = \|\widehat\Gamma - \Gamma\|_{\op}.
\]
Finite samples inevitably introduce estimation error, which can distort the recovered subspace.
We measure subspace discrepancy using the principal angle distance $\sin\Theta(\widehat U,U)$, where $\widehat U$ is the top-$k$ eigenspace of $\widehat\Gamma$.  
Understanding how Fisher estimation error affects probe performance is critical for analyzing stability.   

\paragraph{Classifier and margin.}
To link subspace errors to classification error, we impose a margin assumption.  
Within $U$, the population risk minimizer among linear classifiers is
\[
    f_\star(X)=\sign(a^\top g(X)), \qquad g(X)=U^\top h(X),
\]
where $a\in\R^k$ minimizes misclassification risk.  
Note that this differs from the global Bayes classifier, which may be nonlinear; here we restrict attention to linear probes within $U$.

We assume a Tsybakov margin condition: there exist $\kappa,C>0$ such that
\[
    \Pr(|a^\top g(X)|\le t)\le Ct^\kappa,\quad \forall t>0 .
\]
This smoothness condition ensures that spectral error translates into label-flipping probability.  
For instance, when $\kappa=1$, the probability mass near the decision boundary grows linearly with the margin width, so classification error scales proportionally with subspace deviation.  
Without such a margin condition, even small spectral deviations might either have no effect or cause unbounded misclassification risk.

\paragraph{Notation.}
Table~\ref{tab:notation} summarizes the central quantities (a complete version appears in App.~\ref{app:formal-theorem}).  

\begin{table}[h]
\centering
\small
\begin{tabular}{lll}
\toprule
Quantity & Meaning & Role in theorem \\
\midrule
$h(X)\in\R^d$ & representation & input features \\
$\Gamma=\E[h(X)h(X)^\top]$ & Fisher operator (uncentered moment) & captures discriminative geometry \\
$U\in\R^{d\times k}$ & top-$k$ eigenspace of $\Gamma$ & candidate task subspace \\
$\gap(\Gamma)$ & eigengap $\lambda_k-\lambda_{k+1}$ & determines whether the subspace is recoverable \\
$\widehat\Gamma$ & empirical Fisher operator & data-based estimate of $\Gamma$ \\
$\Delta=\|\widehat\Gamma-\Gamma\|_{\op}$ & operator error & finite-sample deviation \\
$\sin\Theta(\widehat U,U)$ & principal angle distance & measures subspace discrepancy \\
$a$ & optimal direction in $U$ & minimizes risk among linear probes \\
$f_\star(X)$ & Bayes-optimal linear probe in $U$ & benchmark classifier \\
\bottomrule
\end{tabular}
\caption{Core notation for probe stability analysis. Each quantity’s role is indicated with respect to the main theorem in Section~\ref{sec:principle}.}
\label{tab:notation}
\end{table}

\medskip

\section{Spectral Identifiability Principle}
\label{sec:principle}

\paragraph{Opening.}
When can probe accuracy itself be trusted, rather than only audited post hoc?  
The answer is spectral: probe stability arises whenever the empirical Fisher operator is accurate enough to preserve the discriminative subspace.  
This yields a compact and verifiable safeguard: \emph{stability holds when the Fisher estimation error falls below the eigengap}.  
We call this rule the \emph{Spectral Identifiability Principle (SIP)}—an \emph{ex-ante} diagnostic that can be checked by inspecting the Fisher eigenspectrum of the representation–label pair.  
While each proof ingredient is classical, their combination into an operational and falsifiable criterion for probe reliability is novel.

\subsection{Condensed Assumptions}
We assume (full details in App.~\ref{app:formal-theorem}):
\begin{itemize}
    \item \textbf{(R) Regularity.} $h(X)$ has bounded moments, and $\Gamma = \E[h(X)h(X)^\top]$ (uncentered moment capturing representation–label geometry) is finite and well-defined.
    \item \textbf{(S) Spectral gap.} The task subspace $U$ is separated by $\gap(\Gamma) = \lambda_k(\Gamma) - \lambda_{k+1}(\Gamma) > 0$.
    \item \textbf{(C) Concentration.} The estimation error $\Delta = \|\widehat\Gamma - \Gamma\|_{\op}$ concentrates around zero with rate $\tilde O(\sqrt{\log d / n})$ under (V).
    \item \textbf{(V) Variance control.} Bounding $\|h(X)\| \le B$ or assuming sub-Gaussian tails ensures concentration of $\widehat\Gamma$.
    \item \textbf{(M) Margin.} A Tsybakov margin condition with exponent $\kappa>0$ ensures that geometric subspace error translates into excess risk.
\end{itemize}

\subsection{Theorem (Spectral Identifiability Principle)}
\begin{theorem}[Sufficient condition for probe stability]\label{thm:sip}
Assume (R), (S), (V), and (M).  
Let $U$ denote the top-$k$ eigenspace of $\Gamma$, and $\widehat U$ its empirical counterpart obtained from $\widehat\Gamma$.  
There exist constants $c,C>0$ (depending on distributional moments and margin parameters) such that, with probability at least $1 - d^{-c}$ over the sampling of $\widehat\Gamma$, the following hold:
\begin{enumerate}
    \item \textbf{Subspace concentration.} The angle between the true and estimated eigenspaces satisfies
    \[
    \sin\Theta(\widehat U, U) \le \frac{\Delta}{\gap(\Gamma)}.
    \]
    \item \textbf{Risk bound.} For the plug-in probe $\widehat f(X)=\sign(a^\top\widehat U^\top h(X))$,
    \[
    \Pr\{\widehat f(X)\neq f_\star(X)\} \le C\,B^\kappa\,\min\!\left\{1,\!\left(\frac{\Delta}{\gap(\Gamma)}\right)^{\!\kappa}\right\}.
    \]
    \item \textbf{Sample complexity.} Matrix concentration yields $\Delta=\tilde O(\sqrt{\log d/n})$.  
    Consequently, if
    \[
    n \gtrsim \gap(\Gamma)^{-2}\log d,
    \]
    then the misclassification error decays at rate $\tilde O(n^{-\kappa/2})$, up to polynomial factors in $B$ and $1/\kappa$.
\end{enumerate}
Hence, SIP provides a verifiable sufficient condition for probe stability:
\[
\text{If } \Delta < \gap(\Gamma),\ \text{then the probe is stable in both subspace and risk.}
\]
\end{theorem}

\paragraph{Proof sketch.}
Step 1 (geometry): Davis–Kahan’s sin$\Theta$ theorem gives $\sin\Theta(\widehat U,U)\!\lesssim\!\Delta/\gap$.\\
Step 2 (variance): Bounding $\|h(X)\|\!\le\!B$ controls margin distortion, implying $|\widehat s(X)-s_\star(X)|\!\lesssim\!B\sin\Theta(\widehat U,U)$.\\
Step 3 (margin): Tsybakov’s condition converts this distortion into excess risk at rate $(\Delta/\gap)^{\kappa}$.\\
Step 4 (concentration): Matrix Bernstein bounds $\Delta=\tilde O(\sqrt{\log d/n})$ under (V), explaining the $\gap(\Gamma)^{-2}$ scaling in sample complexity.\\
A full statement and proofs appear in Appendix~\ref{app:formal-theorem}.
\qed


\section{Interpreting and Applying the Spectral Identifiability Principle}
\label{sec:sip-practice}

The Spectral Identifiability Principle (SIP), derived in \Cref{sec:principle}, provides a falsifiable spectral rule for probe stability by comparing the Fisher estimation error ($\widehat{\Delta}$) with the eigengap ($\widehat{\text{gap}}$).  
Rather than guaranteeing success, SIP predicts when instability is likely to arise and offers a measurable threshold for diagnosing representation drift.

\paragraph{Geometric interpretation.}
The empirical Fisher operator $\widehat{\Gamma}$ approximates its population counterpart $\Gamma$ with operator deviation $\Delta = \|\widehat{\Gamma}-\Gamma\|_{\op}$.  
When this deviation remains small relative to the eigengap, the estimated subspace $\widehat{U}$ stays aligned with the discriminative directions $U$, preserving decision boundaries despite sampling noise.  
Geometrically, the eigengap acts as a safety margin separating stable and unstable probe regimes.

\paragraph{Finite-sample behavior.}
Matrix concentration bounds imply $\Delta = \tilde{O}(\sqrt{\log d / n})$, so increasing sample size directly improves subspace accuracy.  
When $n$ exceeds a threshold proportional to $\gap(\Gamma)^{-2}\log d$, the expected misclassification risk decays at rate $\tilde{O}(n^{-\kappa/2})$, linking statistical efficiency with spectral geometry.

\paragraph{Practical use.}
In practice, $\widehat{\Delta}$ and $\widehat{\text{gap}}$ can be computed empirically across layers or checkpoints.  
Tracking their ratio provides a simple stability indicator: when $\widehat{\Delta} \ge \widehat{\text{gap}}$, the probe may no longer reflect meaningful structure.  
This diagnostic can complement conventional probe evaluations, offering a quantitative early-warning signal for instability.  
Appendix~\ref{sec:sip-practice-appendix} provides pseudocode and implementation details.

\paragraph{Limitations.}
SIP provides a sufficient but not necessary condition.  
Instability is likely but not guaranteed when $\Delta > \gap(\Gamma)$—for example, if anisotropic noise aligns with non-discriminative directions, subspace stability may persist despite a small eigengap.  
Estimating $\widehat{\Gamma}$ in large networks can also be computationally demanding, requiring stochastic or low-rank approximations.  
Thus, SIP is best viewed as an interpretable, falsifiable diagnostic rather than a universal guarantee. In large networks, estimating $\widehat{\Gamma}$ can be accelerated via randomized SVD or mini-batch Fisher approximations without affecting the spectral ratio, making SIP practical for modern architectures.

\section{Experiments}
\label{sec:experiments}

\paragraph{Purpose.}  
While previous probing techniques evaluate learned representations only after training, they often lack a principled, verifiable condition for probe reliability.  
The Spectral Identifiability Principle (SIP) provides a quantitative and verifiable method for anticipating probe instability and clarifying when probe-based evaluations may become unreliable.
Our experiments test SIP as a diagnostic that identifies instability and delineates the conditions under which probe stability breaks down.  
Specifically, we demonstrate how SIP predicts phase transitions, stability thresholds, and the impact of sample size on probe performance.  
Using analytically tractable Gaussian and Student-$t$ mixtures, we compute Fisher quantities ($\Gamma$, $\widehat\Gamma$, $\Delta$, margins) exactly, ensuring that results are free from approximation errors and cleanly validate SIP's theoretical predictions. Table~\label{tab:notation} summarizes the synthetic configurations (distribution type, sample size range, and optimal $q^\star$), complementing Figures~1–3.

\subsection{Experiment Setup and Methodology}
\label{sec:methodology}

The goal of our experiments is to evaluate the reliability of probes using the SIP framework under different conditions, including various sample sizes and the effects of clipping on stability. We conduct four complementary experiments: 

1. \textit{Subspace Concentration}: We calculate the principal angle distance $\sin\Theta(\widehat{U}, U)$ between the true and estimated subspaces. \newline
2. \textit{Misclassification Risk}: We measure the misclassification error of a linear probe classifier relative to the Bayes-optimal classifier. \newline
3. \textit{Clipping Effect}: We investigate how clipping influences Fisher estimation and probe performance, particularly in heavy-tailed distributions. \newline
4. \textit{Sample Complexity}: We examine how $\Delta$ scales with sample size $n$ and validate the $1/\sqrt{n}$ and $1/n$ scaling laws.

The Gaussian distribution serves as the baseline, while the Student-$t$ distribution models heavy-tailed data, commonly seen in real-world datasets. We perform experiments with both distributions using \textit{sample sizes ($n$)} ranging from 100 to 1000. For the clipping experiments, we explore quantiles ranging from 0.40 to 0.995. Clipping is implemented by truncating extreme probe outputs to limit the influence of outliers, which helps reduce the influence of outliers on Fisher estimation.

\subsection{Geometry: Eigengap as an Anchor}
\emph{(S) Subspace concentration.}  
To validate SIP’s prediction that subspace stability is governed by the spectral gap, we compute the principal angle distance $\sin\Theta(\widehat{U}, U)$ as a function of $\Delta/\mathrm{gap}(\Gamma)$. As shown in Figure~\ref{fig:geometry}(a), all data points lie beneath the line $y = x$, confirming that subspace stability follows the eigengap criterion predicted by Theorem~\ref{thm:sip}.

\emph{(H) Risk bound.}  
In Figure~\ref{fig:geometry}(b), we plot the misclassification risk as a function of sample size $n$. As expected from Theorem~\ref{thm:sip}, the risk exhibits a sharp \textit{phase transition} when $\Delta$ approaches $\mathrm{gap}(\Gamma)$. The transition emerges near $n \approx 500$, where the misclassification risk increases dramatically. This sharp increase signifies that the model becomes unreliable once the Fisher error ($\Delta$) approaches the spectral gap. In practical terms, this phase transition indicates that when the sample size is insufficient, the probe’s performance deteriorates significantly, underscoring the importance of a sufficient sample size to maintain model stability.

\begin{figure}[H]
  \centering
    \begin{subfigure}{0.495\linewidth}
    \includegraphics[width=\linewidth]{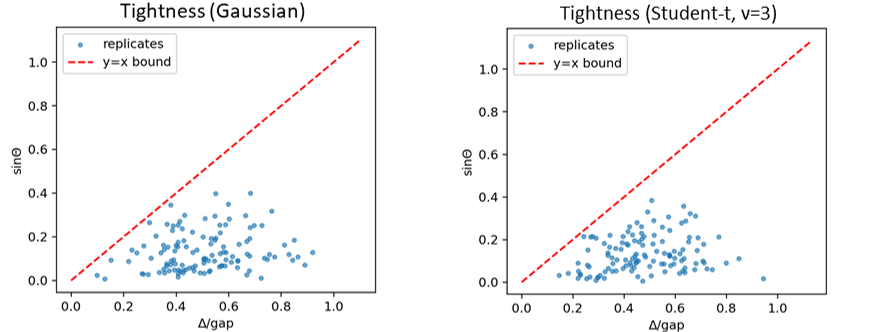}
    \caption{Scatter plot: $\sin\Theta$ vs $\Delta/\mathrm{gap}$ lies beneath $y = x$.}
  \end{subfigure}%
  \begin{subfigure}{0.495\linewidth}
    \includegraphics[width=\linewidth]{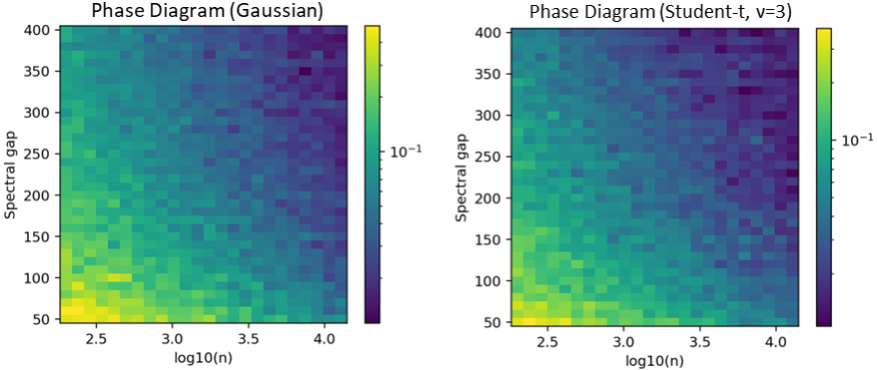}
    \caption{Phase diagram: Risk drops when $\Delta \approx \mathrm{gap}$.}
  \end{subfigure}
  \caption{Geometry. Subspace stability is governed by the spectral threshold in (S).}
  \label{fig:geometry}
\end{figure}

\subsection{Variability: Clipping as a Stabilizer}
In this experiment, we analyze how clipping affects Fisher estimation error $\Delta$ and its role in stabilizing probe performance. In Gaussian distributions, clipping has minimal impact, but in heavy-tailed distributions, clipping mitigates extreme outliers and improves the conditioning of the Fisher estimate. As shown in Figure~\ref{fig:clipping}, excessive clipping reduces the effective gap, while insufficient clipping inflates $\Delta$. An optimal clipping quantile $q^\star$ emerges in the Student-$t$ regime, confirming that clipping stabilizes Fisher estimation in heavy-tailed distributions by controlling the variability of the scores.

The optimal clipping quantile $q^\star$ provides a tool for fine-tuning probes for real-world heavy-tailed data distributions. As shown in Figure~\ref{fig:clipping}, for the Gaussian distribution, clipping has a negligible effect on $\Delta/\mathrm{gap}$, but for the Student-$t$ distribution, clipping significantly impacts the probe's performance, with $q^\star$ shifting as clipping increases. This is a key feature of heavy-tailed distributions, where clipping mitigates outliers and stabilizes the model.

The \textit{bias-variance trade-off} in heavy-tailed distributions is central to this result. Clipping reduces variance by limiting the influence of outliers but introduces some bias by truncating extreme values. SIP quantifies this trade-off and helps optimize the clipping parameter to balance bias and variance. This result provides a principled guideline for tuning probes under heavy-tailed noise.

\begin{figure}[H]
  \centering
  \begin{subfigure}{0.495\linewidth}
    \includegraphics[width=\linewidth]{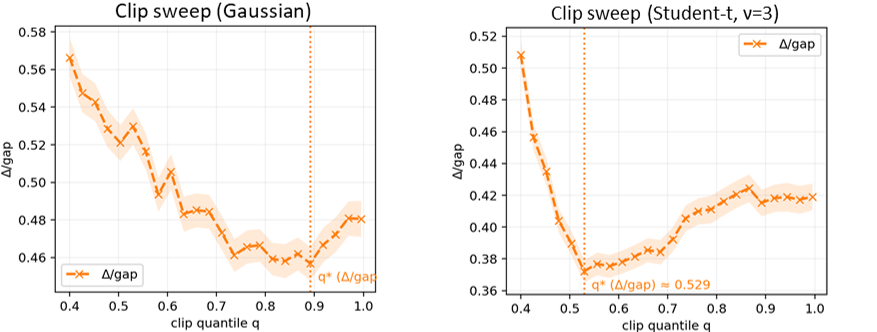}
    \caption{Clipping sweep: Gaussian remains flat; Student-$t$ shows curvature.}
  \end{subfigure}%
  \begin{subfigure}{0.495\linewidth}
    \includegraphics[width=\linewidth]{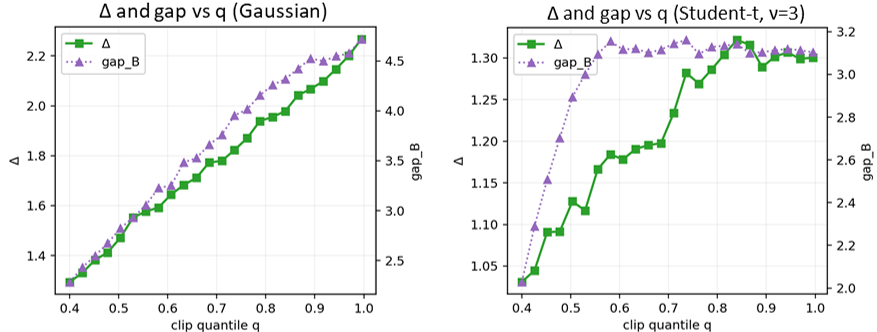}
    \caption{Bias–variance trade-off: Interior optimum $q^\star$.}
  \end{subfigure}
  \caption{Variability. Clipping has negligible effect in Gaussian but stabilizes in heavy-tailed regimes.}
  \label{fig:clipping}
\end{figure}

\subsection{Probability and Sample Complexity}
Margins convert geometric distortion into misclassification risk, while concentration controls how $\Delta$ shrinks with sample size $n$.
As shown in Figure~\ref{fig:margin_concentration}, Gaussian margins decay steeply, and Fisher error follows the expected $1/\sqrt{n}$ and $1/n$ scaling laws, yielding the $n^{-\kappa/2}$ risk rate.
In contrast, Student-$t$ margins decay more slowly, and variance inflation disrupts the expected sub-Gaussian scaling. These results validate both the probabilistic margin-to-risk conversion and the concentration mechanism described in Theorem~\ref{thm:sip} \textit{Sample Complexity}.

As shown in Figure~\ref{fig:margin_concentration}(b), the Gaussian model’s risk decreases as expected according to the $1/\sqrt{n}$ scaling law, whereas the Student-$t$ model’s risk decreases more slowly due to the heavy tail. This highlights how the influence of tail heaviness slows the rate at which risk decays with sample size, demonstrating SIP’s robustness in handling non-Gaussian noise.

\begin{figure}[H]
  \centering
  \begin{subfigure}{0.495\linewidth}
    \includegraphics[width=\linewidth]{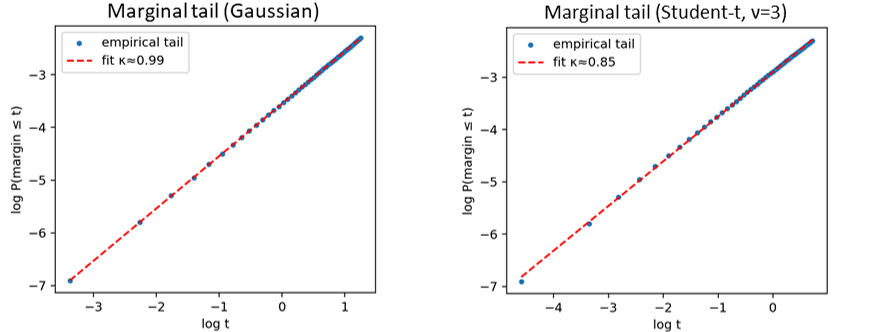}
    \caption{Margin tails: Gaussian decays steeply; Student-$t$ remains flat.}
  \end{subfigure}%
  \begin{subfigure}{0.495\linewidth}
    \includegraphics[width=\linewidth]{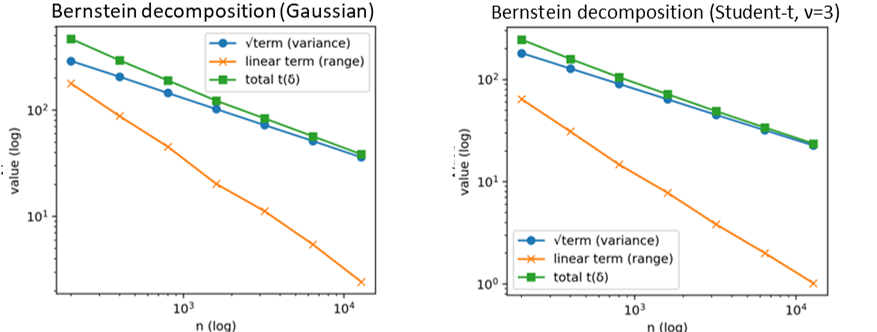}
    \caption{Fisher concentration: $1/\sqrt{n}$ scaling vs. heavy-tail breakdown.}
  \end{subfigure}
  \caption{Probability and sample complexity. $\kappa$ and concentration jointly govern stability.}
  \label{fig:margin_concentration}
\end{figure}

\subsection{Heavy-Tailed Extension: The Sweet Spot}
In heavy-tailed regimes, the classical Bernstein–Davis–Kahan line degenerates, but clipping reveals a robust interior optimum $q^\star$. Figure~\ref{fig:clipping} illustrates this phenomenon, showing that $\sin\Theta$ is minimized at intermediate quantiles and that the location of $q^\star$ remains stable across sample sizes. Formally, we observe that $\sin\Theta(q) \approx O\left(\frac{\sqrt{v(q)/n}}{\mathrm{gap}_B(q)}\right)$, where aggressive clipping reduces variance $v(q)$ while weakening the effective gap $\mathrm{gap}_B(q)$, leading to an optimal trade-off. This behavior is absent in Gaussian distributions, highlighting clipping’s unique effectiveness under heavy-tailed regimes.

The "sweet spot" revealed by clipping provides a practical tool for balancing bias and variance in heavy-tailed settings. This is especially important in real-world datasets, where extreme values or outliers can heavily influence probe performance. In practice, this sweet spot provides a tunable operating point for balancing bias and variance in heavy-tailed data.

\paragraph{Robustness.}  
To ensure the robustness of the sweet spot phenomenon, we repeated all experiments with multiple random seeds (100 seeds per experiment). The location of $q^\star$ consistently fell within the $0.51$–$0.53$ quantile range, indicating the stability of this phenomenon. Additionally, this robustness was observed to extend across different values of $\nu$ in the Student-$t$ distribution, further confirming the consistency and reliability of the results.

\paragraph{Limitation and Insight.}  
While our experiments rely on analytically tractable score families, this choice is deliberate: it enables exact Fisher computation and isolates theoretical mechanisms without confounds.
Although this setting may appear limited for real-world applications, it offers a clean validation of Theorem~\ref{thm:sip}.
Future work will extend SIP to complex distributions and real datasets—such as deep neural representations—where distributional shifts and noise introduce additional challenges. While validated on controlled synthetic settings, the same spectral diagnostic can be directly applied to frozen representations in large models (e.g., BERT or ResNet) using the practical guide in Appendix~\ref{sec:sip-practice-appendix}.


\section{Discussion and Future Directions}
\label{sec:discussion}

\paragraph{Empirical reflections.}  
Our experiments reveal a distinct phase transition around a critical sample size of $n \approx 500$, where misclassification risk rises sharply and subspace stability begins to degrade.  
This transition aligns with SIP’s theoretical prediction: once the Fisher estimation error approaches the eigengap, the discriminative subspace becomes unstable.  
The controlled synthetic setting allows this correspondence to be observed precisely, providing a rare instance where theoretical and empirical boundaries coincide.  
Although such transitions are expected to be less sharp in large-scale neural networks, the result offers an interpretable baseline for understanding stability loss in representation learning. SIP complements existing diagnostics such as MDL and variance-based probing by offering an explicit spectral threshold rather than a post-hoc score.

\paragraph{Conceptual advantage over prior diagnostics.}  
Existing probing diagnostics typically evaluate representations \textit{after} training—quantifying information content or causal dependence once the model is frozen.  
SIP differs conceptually by providing an \textit{ex-ante spectral criterion} that can be verified before deployment.  
This shift from post-hoc measurement to proactive diagnosis reframes probing as a theoretically grounded stability check rather than a heuristic test.  
While empirical comparison with individual methods is beyond the scope of this work, SIP is complementary in spirit: it offers a minimal, verifiable signal that can guide when further empirical evaluation is warranted.

\paragraph{Extending SIP beyond controlled regimes.}  
Adapting SIP to complex architectures remains an open direction.  
In convolutional or recurrent networks, feature dependencies across space or time may alter the Fisher geometry and affect eigengap estimation.  
Understanding how local curvature or hierarchical parameterization modifies spectral identifiability is an important step toward extending SIP to modern deep networks.  
Another open question concerns non-Gaussian and heavy-tailed regimes—common in domains such as healthcare or finance—where Fisher estimates can be distorted by outliers.  
Developing robust spectral estimators or regularized eigengap surrogates may improve stability under such conditions.

\paragraph{Outlook and broader impact.}  
SIP bridges spectral perturbation theory with the practical problem of probe reliability.  
By framing stability through measurable spectral quantities, it provides a small but concrete step toward interpretable and verifiable representation learning.  
Beyond controlled studies, SIP could serve as an analytical layer within real-world systems:  
monitoring layer-wise stability in large language models during fine-tuning,  
identifying representation collapse in vision encoders,  
or flagging unreliable feature separation in safety-critical applications such as healthcare diagnostics.  
Integrating SIP-based metrics into training pipelines could enable early detection of instability, providing a lightweight safeguard for interpretability and trustworthiness in modern neural systems.

\section{Conclusion}
\label{sec:conclusion}

We introduced the Spectral Identifiability Principle (SIP), a framework that formalizes probe reliability through a verifiable spectral condition.  
Unlike traditional methods that assess stability only after training, SIP offers a sufficient criterion for diagnosing feature separability before deployment.  
By linking eigengap geometry to finite-sample behavior, SIP provides a concrete bridge between spectral theory and practical interpretability diagnostics.

Our theoretical analysis, supported by controlled empirical studies, demonstrates that spectral degradation—when the Fisher estimation error approaches the eigengap—serves as a clear indicator of impending instability.  
This understanding enables more principled evaluation of learned representations and may mitigate feature entanglement or overfitting in practice.  
While our validation focuses on analytically tractable settings, the same spectral reasoning can inform diagnostics in large-scale architectures such as BERT or ResNet.  
We view SIP as a step toward verifiable interpretability, offering a foundation for future extensions to nonlinear, hierarchical, and dynamically trained models.

\bibliographystyle{unsrtnat} 
\bibliography{main_bib}

\newpage


\appendix
\section{Formal Statements and Proofs}

\label{app:formal-theorem}

\subsection{Appendix: Formal statement of Theorem A}

\begin{theorem}[Subspace concentration and misclassification: assumptions and conclusions]
\label{thm:formal-A}

\textbf{Assumptions.}

\paragraph{(A1) Regularity.}
For each $e\in\mathcal E$, the conditional log-density $\log p_e(x\mid z)$ is $C^{r}$ in $z$ for some $r\ge 2$, 
with derivatives up to order $r$ dominated by an integrable envelope.
Consequently, differentiation and expectation commute, so the score $s_e(x,z)$, Hessian $H_e(x,z)$, and the Fisher objects
\[
F_e(z):=\E_{X\sim p_e(\cdot\mid z)}[s_e(X,z)s_e(X,z)^\top], 
\qquad
\Gamma_e:=\E_{Z\sim\omega}[F_e(Z)]
\]
are well-defined and finite.

\paragraph{(A2) Identifiability via eigengap.}
Fix a reference environment $e_0$ and write $\Gamma:=\Gamma_{e_0}$.  
Let $U\in\R^{d\times k}$ be the eigenspace corresponding to a chosen spectral block.  
We require that this block is separated from the rest of the spectrum:
\[
\operatorname{gap}(\Gamma)\ :=\ 
\min\big\{\,\lambda_{\text{in-block}}-\lambda_{\text{out-block}}\,\big\}\ >\ 0.
\]
This guarantees that $U$ is well-defined (up to block rotations).

\paragraph{(A3) Estimator.}
From $n$ i.i.d.\ samples $(X_i,Z_i)$ under $e_0$, form a symmetric estimator $\widehat\Gamma$ of $\Gamma$.  
For example, the sample mean
\[
\widehat\Gamma\ :=\ \frac{1}{n}\sum_{i=1}^n s_{e_0}(X_i,Z_i)s_{e_0}(X_i,Z_i)^\top.
\]

\paragraph{(A4) Notation: estimation error.}
Define the estimation error
\[
\Delta := \|\widehat\Gamma-\Gamma\|_{\op}.
\]
Deterministic conclusions will be stated in terms of $\Delta$.  
No additional assumption is made here.

\paragraph{(A5) Tail control for scores (concentration).}
This assumption controls $\Delta$ from (A4).  
Assume the score $h(X):=s_{e_0}(X,Z)$ is clipped at some radius $B<\infty$ and the clipped version is sub-exponential, i.e.
\[
\E\exp(\|h(X)\|/\alpha)\ \le\ 2
\quad\text{for some }\alpha>0.
\]
This ensures $\|Y_i\|_{\op}\le R<\infty$ for $Y_i=s_{e_0}(X_i,Z_i)s_{e_0}(X_i,Z_i)^\top-\Gamma$, 
and a finite variance proxy $v:=\|\E[Y_1^2]\|_{\op}$, enabling matrix Bernstein inequalities.

\paragraph{(A6) Margin condition (classification).}
We take $f_\star$ to be the Bayes-optimal linear classifier in the discriminative subspace $U$, with decision score
\[
s_\star(X) := a^\top U^\top h(X), \qquad f_\star(X) = \sign(s_\star(X)),
\]
for some unit vector $a \in \mathbb{R}^k$. 
Assume a Tsybakov-style margin condition: there exist $\kappa, C > 0$ such that, for all $t>0$,
\[
\Pr\big(|s_\star(X)| \le t\big) \; \le \; C\, t^{\kappa}.
\]
This standard assumption in statistical learning ensures that small subspace errors translate smoothly into small excess misclassification risk rather than causing brittle failures.

\medskip
\textbf{Conclusions.}

\paragraph{(S) Subspace concentration (deterministic).}
Under (A1)–(A3), with $\Delta$ from (A4), for $\widehat U$ the $k$-eigenspace of $\widehat\Gamma$ aligned with $U$,
\[
\sin\Theta(\widehat U,U)\ \le\ \frac{\Delta}{\operatorname{gap}(\Gamma)}.
\]

\paragraph{(H) High-probability risk bound.}
Under (A1)–(A6), the margin argument gives
\[
\Pr\{\widehat f(X)\neq f_\star(X)\}
\;\le\;
C\Big((1+\sqrt2)B \cdot \min\{1,\,\tfrac{\Delta}{\operatorname{gap}(\Gamma)}\}\Big)^{\kappa}.
\]
Matrix concentration yields $\Delta = O(\sqrt{\tfrac{\log d}{n}})$ with high probability, 
so once $n\ \gtrsim\ \frac{v}{\operatorname{gap}(\Gamma)^2}\,\log\!\frac{d}{\delta},$ 
\[
\Pr\{\widehat f(X)\neq f_\star(X)\}
\;=\;\tilde O\!\big(n^{-\kappa/2}\big),
\]
\emph{explicitly governed by the eigengap $\operatorname{gap}(\Gamma)$ that secures spectral identifiability.}

\end{theorem}

\medskip

\begin{proof}[Proof of Theorem~\ref{thm:formal-A}]
We proceed in four steps.

\medskip
\noindent\textbf{Step 1: subspace perturbation (deterministic).}
Let $P:=UU^\top$ and $\widehat P:=\widehat U\widehat U^\top$ be the orthogonal projectors onto the $k$-dimensional target block of $\Gamma$ and $\widehat\Gamma$ respectively.  
By (A2)–(A3), $\Gamma$ has a spectral gap $\operatorname{gap}(\Gamma)>0$ separating the chosen block, and $\widehat\Gamma$ is symmetric.  
Davis–Kahan (sin–$\Theta$ form) gives
\[
\|\sin\Theta(\widehat U,U)\|\ \le\ \frac{\|\widehat\Gamma-\Gamma\|_{\op}}{\operatorname{gap}(\Gamma)}
\ =\ \frac{\Delta}{\operatorname{gap}(\Gamma)}.
\]
This yields conclusion (S).

\medskip
\noindent\textbf{Step 2: projecting the score—coordinate perturbation.}
Write $g(X):=U^\top h(X)\in\R^k$ and $\widehat g(X):=\widehat U^\top h(X)\in\R^k$ for the population vs.\ estimated subspace coordinates of the (clipped) score $h(X)$.  
We use the following standard bound for principal vectors.

\begin{lemma}[coordinate perturbation]
\label{lem:coord}
There exists an orthogonal $Q\in\R^{k\times k}$ such that for all $x$,
\[
\|\widehat g(x)-Q^\top g(x)\|
\ \le\ \sqrt2\,\|h(x)\|\,\sin\Theta(\widehat U,U),
\qquad
\|Q^\top g(x)-g(x)\|
\ \le\ \|h(x)\|\cdot \sin\Theta(\widehat U,U).
\]
Consequently, for any $a\in\R^k$ with $\|a\|=1$,
\[
\big|\,a^\top\widehat g(x)-a^\top g(x)\,\big|
\ \le\ (1+\sqrt2)\,\|h(x)\|\,\sin\Theta(\widehat U,U).
\]
\end{lemma}

\begin{proof}
Let $Q$ be the Procrustes aligner from the polar decomposition $U^\top\widehat U=QR$.  
Decompose
\[
\widehat U-UQ=(I-UU^\top)\widehat U+U(U^\top\widehat U-Q)=:A+B .
\]
By geometry of principal angles, $\|A\|_{\op}=\|\sin\Theta(\widehat U,U)\|$ and 
$\|B\|_{\op}=\|R-I\|_{\op}\le \|\sin\Theta(\widehat U,U)\|$.  
Since $A\perp B$, 
\[
\|\widehat U-UQ\|_{\op}\le \sqrt{\|A\|_{\op}^2+\|B\|_{\op}^2}
\le \sqrt2\,\|\sin\Theta(\widehat U,U)\|.
\]
Hence 
\[
\|\widehat g(x)-Q^\top g(x)\| = \|(\widehat U-UQ)^\top h(x)\|
\le \|\widehat U-UQ\|_{\op}\|h(x)\|
\le \sqrt2\,\|h(x)\|\sin\Theta.
\]
Moreover $\|Q-I\|_{\op} = \max_i|1-\cos\theta_i| \le \max_i\sin\theta_i = \|\sin\Theta\|$, giving 
$\|Q^\top g(x)-g(x)\|\le \|h(x)\|\sin\Theta$.  
The last inequality in the lemma then follows by the triangle inequality. \qedhere
\end{proof}

Using Lemma~\ref{lem:coord} and the trivial bound $\sin\Theta\le 1$,
\begin{equation}
\label{eq:coord-to-gap}
\big|\,a^\top\widehat U^\top h(X)-a^\top U^\top h(X)\,\big|
\ \le\ (1+\sqrt2)\,\|h(X)\|\,\min\Big\{1,\ \frac{\Delta}{\operatorname{gap}(\Gamma)}\Big\}.
\end{equation}

\medskip
\noindent\textbf{Step 3: margin-to-risk reduction.}
Define the oracle score $s_\star(X):=a^\top U^\top h(X)$ and the plug-in score $\widehat s(X):=a^\top \widehat U^\top h(X)$.  
By definition, $f_\star(X)=\sign(s_\star(X))$ and $\widehat f(X)=\sign(\widehat s(X))$.  
By the usual margin argument,
\[
\{\widehat f(X)\neq f_\star(X)\}\ \subseteq\ \big\{|s_\star(X)|\le |\widehat s(X)-s_\star(X)|\big\}.
\]
Therefore, for any $\tau>0$,
\[
\Pr\{\widehat f(X)\neq f_\star(X)\}
\ \le\
\Pr\!\big\{|s_\star(X)|\le \tau\big\}
\ +\
\Pr\!\big\{|\widehat s(X)-s_\star(X)|>\tau\big\}.
\]
By the Tsybakov margin (A6) with $\|a\|=1$,
\[
\Pr\!\big\{|s_\star(X)|\le \tau\big\}\ \le\ C\,\tau^\kappa.
\]
Choose
\[
\tau\ :=\ (1+\sqrt2)\,B\cdot \min\Big\{1,\ \frac{\Delta}{\operatorname{gap}(\Gamma)}\Big\}.
\]
Since the clipped score obeys $\|h(X)\|\le B$ under (A5), inequality \eqref{eq:coord-to-gap} yields $|\widehat s(X)-s_\star(X)|\le\tau$ almost surely, hence the second probability above is $0$.  
Thus, \emph{deterministically},
\begin{equation}
\label{eq:det-risk}
\Pr\{\widehat f(X)\neq f_\star(X)\}
\ \le\
C\Big((1+\sqrt2)B\cdot \min\{1,\Delta/\operatorname{gap}(\Gamma)\}\Big)^{\kappa}.
\end{equation}

\medskip
\noindent\textbf{Step 4: concentration for $\Delta$.}
Let $Y_i:=s_{e_0}(X_i,Z_i)s_{e_0}(X_i,Z_i)^\top-\Gamma$.  
By (A1) and (A5) the clipped score has $\|s_{e_0}(X_i,Z_i)\|\le B$ almost surely, hence
\[
\|Y_i\|_{\op}\ \le\ \|s_i s_i^\top\|_{\op}+\|\Gamma\|_{\op}
\ \le\ B^2+\|\Gamma\|_{\op}
\ :=\ R,
\]
and the matrix variance proxy $v:=\|\E Y_1^2\|_{\op}$ is finite.  
Applying the matrix Bernstein inequality \cite[Theorem~6.1]{tropp2012user}, for any $\delta\in(0,1)$, with probability at least $1-\delta$,
\[
\Delta\ =\Big\|\frac1n\sum_{i=1}^n Y_i\Big\|_{\op}
\ \le\
t(\delta)
\ :=\
\sqrt{\frac{2v\log(2d/\delta)}{n}}
\ +\ \frac{2R\log(2d/\delta)}{3n}.
\]

The factors in $t(\delta)$ follow directly from \cite[Theorem~6.1]{tropp2012user}:  the variance term $\sqrt{2v\log(2d/\delta)/n}$ comes from $\sigma^2=nv$ after rescaling by $1/n$, 
the linear term $2R\log(2d/\delta)/(3n)$ from the boundedness $R$, 
and the $\log(2d/\delta)$ from the $d$ prefactor and symmetrization. \\

Combining this event with \eqref{eq:det-risk} and using the monotonicity of $x\mapsto \min\{1,x\}$ gives, with probability $\ge 1-\delta$,
\[
\Pr\{\widehat f(X)\neq f_\star(X)\}
\ \le\
C\Big((1+\sqrt2)B\cdot \min\{1,\ t(\delta)/\operatorname{gap}(\Gamma)\}\Big)^{\kappa}.
\]

\medskip

\noindent\textbf{Step 5: soft-rate sample complexity.}
Recall from Step~4 that, with probability at least $1-\delta$,
\[
\Delta \;\le\; t(\delta)\ :=\ t_1+t_2, \qquad
t_1:=\sqrt{\frac{2v\log(2d/\delta)}{n}},\quad
t_2:=\frac{2R\log(2d/\delta)}{3n}.
\]
We show that for sufficiently large $n$, the variance term dominates and the ``min'' is linearized.

\emph{(a) Variance dominates $t_2\le t_1$.}
This is equivalent to
\[
\frac{2R\log(2d/\delta)}{3n}\ \le\ \sqrt{\frac{2v\log(2d/\delta)}{n}}
\ \Longleftrightarrow\
n\ \ge\ \frac{2R^2}{9v}\,\log\!\frac{2d}{\delta}\ =:\ n_{\mathrm{var}}.
\]
Under $n\ge n_{\mathrm{var}}$ we have $t(\delta)=t_1+t_2\le 2t_1$.

\emph{(b) Entering the gap-controlled phase.}
It suffices to enforce $t(\delta)\le \tfrac12\,\operatorname{gap}(\Gamma)$ so that
$\min\{1,\ t(\delta)/\operatorname{gap}(\Gamma)\} = t(\delta)/\operatorname{gap}(\Gamma)$.
Using $t(\delta)\le 2t_1$ from (a), it is enough to require
\[
2t_1\ \le\ \tfrac12\,\operatorname{gap}(\Gamma)
\ \Longleftrightarrow\
n\ \ge\ \frac{32\,v}{\operatorname{gap}(\Gamma)^2}\,\log\!\frac{2d}{\delta}
\ =:\ n_{\mathrm{gap}}.
\]

On the event of Step~4 and for $n\ge \max\{n_{\mathrm{var}},\,n_{\mathrm{gap}}\}$, plugging the above into the high-probability risk bound yields
\[
\Pr\{\widehat f(X)\neq f_\star(X)\}
\ \le\
C\!\left((1+\sqrt2)B\cdot \frac{2t_1}{\operatorname{gap}(\Gamma)}\right)^{\kappa}
\ =\
C'\!\left(\frac{B\sqrt{v}}{\operatorname{gap}(\Gamma)}\right)^{\!\kappa}
\left(\frac{\log(2d/\delta)}{n}\right)^{\!\kappa/2}.
\]
Absorbing logarithmic and geometry/clipping factors into $\tilde O(\cdot)$ gives the desired soft-rate
\[
\Pr\{\widehat f(X)\neq f_\star(X)\}\ =\ \tilde O\!\big(n^{-\kappa/2}\big).
\]
\qedhere

\end{proof}\newpage

\section{Appendix: Practical Guide for Neural Networks (SIP in Practice)}
\label{sec:sip-practice-appendix}

\subsection{Pipeline Overview (Pseudocode)}
\begin{algorithm}[H]
\caption{SIP Practical Guide for Neural Networks}
\label{alg:sip-guide}
\begin{algorithmic}[1]
\State \textbf{Input:} Dataset $\{(x_i, y_i)\}_{i=1}^n$, Frozen model, Target layer, Probe dimension $k$
\State \textbf{Output:} $\widehat{\text{gap}}$, $\widehat{\Delta}$, SIP Verdict, Clipping parameter $q^\star$, $n_{\min}$

\State \textbf{Step 1: Feature Extraction}
\State Set model to evaluation mode
\State Collect features $h_i = h(x_i)$ from target layer for $i=1,2,\dots,n$
\State Check Assumption (R): Compute kurtosis and tail index for $\{h_{ij}\}$ or $\|h_i\|_2$
\If {Kurtosis $\gg 3$ or tail index $\alpha < 4$}
    \State Flag for heavy tails, apply winsorization or clipping
\EndIf

\State \textbf{Step 2: Variance Control (Optional)}
\State Detect heavy tails using kurtosis or tail tests
\If{Heavy tails detected}
    \State Apply clipping or winsorization to features $h_i$
\EndIf

\State \textbf{Step 3: Empirical Fisher Estimate}
\State Compute $\widehat{\Gamma} = \frac{1}{n} \sum_{i=1}^{n} h_i h_i^\top$
\State Use randomized SVD to compute top-$k$ eigenspectrum

\State \textbf{Step 4: Check Eigengap}
\State Calculate $\widehat{\text{gap}} = \hat{\lambda}_k - \hat{\lambda}_{k+1}$
\State Check Assumption (G): Ensure $\widehat{\text{gap}} > 0$
\If {Weak gap detected}
    \State Increase sample size $n$ or adjust $k$
    \State Optionally, apply mild ridge regularization to separate bulk
\EndIf

\State \textbf{Step 5: Estimate Fisher Error Proxy $\widehat{\Delta}$}
\State Split dataset into $A$ and $B$ (stratified sampling)
\State Compute $\widehat{\Gamma}_A$ and $\widehat{\Gamma}_B$
\State Estimate $\widehat{\Delta} = \frac{1}{2} \|\widehat{\Gamma}_A - \widehat{\Gamma}_B\|_{\text{op}}$
\State Check Assumption (C): Verify scaling of $\widehat{\Delta}$ with sample size

\State \textbf{Step 6: SIP Decision}
\If {$\widehat{\Delta} < \widehat{\text{gap}}$}
    \State SIP Pass
\Else
    \State SIP Fail
\EndIf

\State \textbf{Step 7: Sweet-Spot Clipping (for Heavy-Tail Only)}
\State Sweep $q \in \{0.80, 0.81, \dots, 0.98\}$ for clipping quantile
\State Recompute $\widehat{\Delta}(q)$ and $\widehat{\text{gap}}(q)$ for each $q$
\State Pick $q^\star = \arg \min_q \frac{\widehat{\Delta}(q)}{\widehat{\text{gap}}(q)}$

\State \textbf{Step 8: Sample Complexity Estimate}
\State Estimate $n_{\min} \approx C \cdot \frac{\log d}{\widehat{\text{gap}}^2}$ with $C$ chosen to satisfy $\widehat{\Delta}(n_{\min})/\widehat{\text{gap}} \lesssim 0.5$

\State \textbf{Step 9: Report}
\State Output spectrum plot, $\widehat{\Delta}/\widehat{\text{gap}}$ ratio, SIP verdict, $q^\star$, $n_{\min}$, and probe accuracy
\end{algorithmic}
\end{algorithm}

\paragraph{Scope.}  
This appendix outlines the steps to apply the Spectral Identifiability Principle (SIP) principle in practice, particularly for frozen neural networks using layer-wise probing. The representation $h(x) \in \mathbb{R}^d$ is extracted from a chosen layer, with dropout off and batch normalization frozen. The task is to assess the reliability of the probe through eigengap and Fisher error.

\subsection{Step-by-Step Guide to SIP for Neural Networks}

\paragraph{Step 1: Feature Extraction and Assumption Check.}  
First, we collect the features from the target layer of the neural network. We then check the data for potential issues, such as heavy-tailed distributions, which could interfere with the analysis. Specifically, we check the kurtosis (a measure of the "pointiness" of the data distribution) and the tail index (which quantifies the extremity of outliers). If these values suggest heavy tails, we apply techniques like \textit{winsorization} or \textit{clipping} to limit the influence of extreme values and stabilize the model. The features are represented as $h(x_i) \in \mathbb{R}^d$, where $h(\cdot)$ denotes the feature extraction function from the target layer.

\paragraph{Step 2: Variance Control (Optional).}  
If the data shows heavy tails, we apply \textit{clipping} (which limits extreme values) or \textit{winsorization} (which replaces extreme values with the nearest valid ones). These steps are optional but recommended because they help reduce the instability in our Fisher estimate $\widehat{\Gamma}$, which is critical for accurate predictions.

\paragraph{Step 3: Empirical Fisher Estimate.}  
We compute the \textit{Fisher estimate} $\widehat{\Gamma}$ by averaging the outer products of the feature vectors. To do this efficiently, we use \textit{randomized singular value decomposition (SVD)}. This method is faster and more suitable for large neural networks, where traditional methods would be computationally expensive. Randomized SVD allows us to focus on the most important features (the top $k$ eigenvectors), reducing the computational cost while capturing essential information.

\paragraph{Step 4: Check Eigengap.}  
Next, we check the \textit{eigengap} $\widehat{\text{gap}}$, which is the difference between the top two eigenvalues of the Fisher operator $\widehat{\Gamma}$. A significant gap means that the most important directions in the data are well-separated and stable. If the gap is small (a "weak gap"), the model may not be reliable. In this case, we either increase the sample size $n$ or adjust the number of eigenvectors \textit{k}. If needed, we apply \textit{ridge regularization} to help improve stability by separating the bulk of the spectrum.

\paragraph{Step 5: Estimate Fisher Error Proxy.}  
To estimate the error in our Fisher estimate, we split the dataset into two parts, \textit{A} and \textit{B}. We calculate the Fisher estimate for each part and compute the difference between them. The error proxy, $\widehat{\Delta}$, tells us how much the estimate varies. If $\widehat{\Delta}$ is small compared to the eigengap $\widehat{\text{gap}}$, it indicates stable performance.

\paragraph{Step 6: SIP Decision.}  
The core of SIP is a simple rule: if the error proxy $\widehat{\Delta}$ is smaller than the eigengap $\widehat{\text{gap}}$, the probe is stable, and we pass the test. Otherwise, we fail the test and may need to adjust the model or collect more data.

\paragraph{Step 7: Sweet-Spot Clipping (for Heavy-Tail Data).}  
For datasets with extreme outliers, we fine-tune the \textit{clipping} process to find the optimal quantile $q^\star$ (e.g., between 0.85 and 0.95). By testing various clipping levels, we can find the "sweet spot" that minimizes the error proxy $\widehat{\Delta}$ relative to the eigengap $\widehat{\text{gap}}$. This step ensures that the clipping helps stabilize the model without introducing too much bias.

\paragraph{Step 8: Sample Complexity Estimate.}  
We estimate the minimum sample size $n_{\min}$ needed to achieve stable performance using the formula:
\[
n_{\min} \approx C \cdot \frac{\log d}{\widehat{\text{gap}}^2},
\]
where $C$ is a constant that depends on the dataset and model architecture. This estimate tells us how many samples are needed to achieve reliable results. We aim for a sample size that ensures the error proxy $\widehat{\Delta}$ is well below the eigengap $\widehat{\text{gap}}$.

\subsection{Minimal Reporting Template}
For reproducibility and clarity, include the following details in your report:
\begin{itemize}
\item Model layer index, dimensions ($d$, $k$), and sample size ($n$); clipping quantile $q^\star$ (if used); ridge $\rho$ (if used).
\item Eigen gap $\widehat{\text{gap}}$, Fisher error proxy $\widehat{\Delta}$, and the ratio $\widehat{\Delta}/\widehat{\text{gap}}$.
\item SIP verdict (Pass/Fail) and estimated minimum sample size $n_{\min}$.
\item Optional: Spectrum plots and visualizations of $\widehat{\Delta}/\widehat{\text{gap}}$ ratio, as well as histograms of eigenvalues and error distributions.
\end{itemize}
This reporting structure will allow others to reproduce your results and understand the stability and reliability of your probe.

\end{document}